\documentclass[11pt]{article}
\PassOptionsToPackage{hyphens}{url}
\usepackage {acl}
\usepackage[T1]{fontenc}
\usepackage[utf8]{inputenc}
\usepackage{times}
\usepackage{latexsym}
\usepackage{microtype}
\usepackage{amsmath,amssymb,amsthm,mathtools}
\usepackage{booktabs}
\usepackage{multirow}
\usepackage{graphicx}
\usepackage{xcolor}
\usepackage{tabularx}
\usepackage{array}
\usepackage{url}
\usepackage{hyperref}
\hypersetup{hidelinks}
\newtheorem{theorem}{Theorem}
\newtheorem{proposition}[theorem]{Proposition}
\newtheorem{lemma}[theorem]{Lemma}
\newtheorem{corollary}[theorem]{Corollary}
\theoremstyle{definition}

\theoremstyle{remark}

\newcommand{\E}{\mathbb{E}}
\newcommand{\Prb}{\mathbb{P}}

\newcommand{\CE}{\operatorname{CE}^{(1)}}
\newcommand{\Brier}{\operatorname{Brier}}
\newcommand{\AUROC}{\operatorname{AUROC}}
\newcommand{\src}{\mathrm{s}}
\newcommand{\tgt}{\mathrm{t}}

\newcommand{\calG}{\mathcal{G}}
\newcommand{\ind}{\mathbf{1}}

\newcommand{\figureasset}[2][]{%
  \IfFileExists{#2}{\includegraphics[#1]{#2}}{%
    \fbox{\parbox[c][32mm][c]{0.90\linewidth}{\centering
    Missing figure asset: \texttt{\detokenize{#2}}}}}}
\title{When Accuracy Gaps Fail to Certify:\\ Auditing Cross-Domain Recalibration of LLM Judges}

\author{
\textbf{Fariya Afrin}\textsuperscript{1}%
\thanks{Equal contribution.}%
\and
\textbf{Ibne Farabi Shihab}\textsuperscript{2}\footnotemark[1]
\thanks{Corresponding author: \texttt{ishihab@iastate.edu}.}
\\[2pt]
\textsuperscript{1}Department of Computer Science, Kalinga Institute of Industrial Technology \\
\textsuperscript{2}Department of Computer Science, Iowa State University \\
\texttt{ishihab@iastate.edu}
}

\begin{document}
\maketitle

\begin{abstract}
A scalar recalibration map fitted for an LLM judge on one task can fail when the task distribution changes, but the source--target accuracy gap is often treated as a proxy for that failure. We test what this gap can predict and what it can certify across thirteen judges, two generators, eight domains, and 1{,}176 predeclared transfers. After accounting for mean score shift, the gap yields a population lower bound on target calibration error, yet identical gaps can induce opposite transfer outcomes. Exact importance weighting recovers target proper loss under covariate shift, so failure of an estimated weighting pipeline does not by itself establish conditional shift. A finite-sample simultaneous lower certificate converts the population bound into a one-sided rejection rule using audit labels disjoint from evaluation outcomes. The leak-free gap correlation is $0.25$ (95\% CI $[-0.09,0.55]$), falls to $0.09$ on the second generator, and does not support a generator-invariant association. The certificate retains nominal coverage but has power $0.13$ even at $m{=}1024$, whereas target-domain temperature scaling with 16 labels reaches harm rate $0.09$, compared with $0.34$ for source-fitted Platt scaling. Accuracy gaps are therefore weak warning signals for scalar probability transfer, not deployment certificates.
\end{abstract}

\section{Introduction}

A judge can preserve the ordering of answers while its reported probabilities
become unusable after a task change. This failure matters whenever a score such
as $0.8$ controls acceptance, review, or deferral rather than serving only as a
ranking statistic. In that setting, calibration requires approximately $80\%$
of the cases assigned probability $0.8$ to be correct. The requirement is local
to a distribution, not an intrinsic property of the judge.

Post-hoc maps such as temperature and Platt scaling can repair probability
quality on held-out source labels \citep{platt1999probabilistic,
guo2017calibration}. Under explicit covariate- or label-shift assumptions,
importance weighting can also identify target risk
\citep{park2020covariate,podkopaev2021labelshift,popordanoska2024lascal}.
LLM judges, however, exhibit overconfidence and unstable uncertainty across
tasks \citep{tian2025overconfidence,sheng2025interval}. A deployment team then
faces a narrower question: which observable signals diagnose scalar
recalibration failure, and when can those signals support a formal rejection?

The absolute source--target accuracy gap is a tempting diagnostic, but it is a
marginal quantity. Calibration depends on the conditional relationship
$\E[Y\mid S]$ between correctness $Y$ and judge score $S$. Equal accuracy gaps
can therefore hide different conditional reliability functions. The gap is
also not label-free because estimating target accuracy consumes target labels.
Any gap-based rule must consequently be compared with fitting a target-domain
calibrator using the same label budget.

This paper develops that comparison for scalar probability transfer. The
population analysis separates the accuracy gap from the shift in the
recalibrated mean score (Corollary~\ref{cor:gap}), proves that the gap is
insufficient (Proposition~\ref{prop:insufficiency}), and states exactly when
importance weighting recovers target proper loss
(Theorem~\ref{thm:iw}). The same lower bound yields a simultaneous finite-sample
certificate that rejects a transferred map only when its target calibration
error must exceed a tolerance (Theorem~\ref{thm:certificate}). Non-rejection is
not approval.

The empirical audit contains $448$ primary transfers across eight judges and
eight domains, $224$ transfers from four additional 7--9B judges, $56$ from one
frontier API judge, and a $448$-transfer second-generator replication. The
leak-free gap association does not survive the cross-generator standard, while
probe-estimated raw target error does (Table~\ref{tab:regression}). Controlled
shift constructions verify the importance-weighting implementation and locate
its natural-domain failure in the observed text representation
(Table~\ref{tab:controls}). The rejection certificate is valid but too weak at
the available audit sizes; sixteen target labels are more useful for direct
temperature scaling (Table~\ref{tab:certificate}).

These results delimit one scientific object: a scalar map $g(S)$ evaluated by
proper loss and calibration error. Concurrent work by overlapping authors
studies set-valued conformal coverage under finite traffic-mixture shift
\citep{anonymous2027cswcp}. It neither evaluates the scalar transfer question
answered here nor shares the stated contributions of this paper.

\section{Related Work}

\subsection{LLM judges and their uncertainty}

LLM-as-a-judge research has primarily studied agreement with humans, ranking
quality, position effects, verbosity bias, and self-preference
\citep{zheng2023judging,liu2023geval,dubois2024alpacafarm,panickssery2024selfpreference}.
Confidence adds a distinct measurement problem. \citet{tian2025overconfidence}
document overconfidence in LLM judges, while \citet{sheng2025interval} construct
conformal intervals for rating-based evaluation. ConfidenceBench evaluates
verbalized probabilities with proper scores across frontier models
\citep{ffrenchconstant2026confidencebench}, while SAJA uses labeled target data
to align multidimensional rubric features to human judgments
\citep{kola2026saja}. These studies establish that confidence quality matters,
but they do not test whether a one-dimensional map fitted on one task transfers
to another under disjoint target evaluation. Evaluator failure has also
been framed as an adversarial-robustness problem: \citet{akter2026proxy}
detect proxy gaming of reward models and LLM judges with invariance-based
stress tests. Our question is complementary: we hold the judge and response
set fixed and study transfer of a scalar post-hoc calibration map across
domains, a probability-level failure mode arising without optimization
pressure against the judge. We report ranking, proper scores, calibration
metrics, and downstream risk separately.

\subsection{Post-hoc calibration and distribution shift}

Temperature and Platt scaling fit low-dimensional maps on held-out labels
\citep{platt1999probabilistic,guo2017calibration}; beta calibration and isotonic
regression provide more flexible alternatives \citep{kull2017beta,
zadrozny2002transforming}. Calibration commonly degrades under distribution
shift \citep{ovadia2019trust}. Under covariate shift, importance weighting can
identify target risk if conditional labels are invariant and target support is
covered by source support \citep{shimodaira2000improving,park2020covariate}.
Under label shift, separate reweighting and calibration procedures are available
\citep{podkopaev2021labelshift,popordanoska2024lascal}. We therefore evaluate
these assumptions as distinct hypotheses; we do not infer conditional shift
merely because one estimated weighting pipeline fails.

\subsection{Measuring calibration}

Fixed-bin ECE is familiar but can be substantially biased and depends on binning
choices \citep{roelofs2022bias}. We use population $L_1$ calibration error for
theory, Brier score and log loss as primary empirical outcomes, and equal-mass
ECE as a secondary diagnostic; debiased and smoothed ECE variants were not
computed and are listed in Appendix~\ref{app:checks} as not run. This
ordering prevents a constant, uninformative predictor from appearing strong
solely because its mean matches the label prevalence.

\subsection{Boundary with concurrent work}

Concurrent work by overlapping authors introduces a conformal prediction-set
construction for traffic-mixture shift \citep{anonymous2027cswcp}. Its output
is a set of admissible labels and its guarantee concerns marginal coverage
under a finite partition with uncertain group masses. This paper instead studies
one-dimensional post-hoc maps, proper-loss transfer, the information carried by
an accuracy gap, and a lower certificate for scalar calibration error. It does
not propose a conformal set, claim label-free target calibration, or reuse the
companion paper's CS-WCP experiments as evidence.

\section{Problem Setup}
\label{sec:setup}

Let $D\in\{\src,\tgt\}$ denote source or target domain, and let
$(X,S,Y)\sim P_D$. Here $X$ is the complete judge input, including the
prompt, candidate answer, and reference; $S\in[0,1]$ is the judge-assigned
probability that the candidate answer is correct (our protocol elicits this
scalar directly; no intermediate binary verdict is produced); and
$Y\in\{0,1\}$ is the reference correctness of that candidate answer.
A post-hoc calibrator $g:[0,1]\to[0,1]$ outputs $Q=g(S)$.

The domain accuracy and recalibrated mean score are
\begin{equation}
\pi_D=\E_D[Y],\qquad \mu_D(g)=\E_D[g(S)].
\end{equation}
The source--target accuracy gap is $A=|\pi_\tgt-\pi_\src|$. Let
$\eta_{D,g}(q)=\E_D[Y\mid g(S)=q]$. We define population $L_1$ calibration
error as
\begin{equation}
\CE_D(g)=\E_D\!\left[\left|\eta_{D,g}(g(S))-g(S)\right|\right].
\label{eq:ce}
\end{equation}
We additionally report
\begin{align}
\Brier_D(g)&=\E_D[(g(S)-Y)^2],\\
\operatorname{LogLoss}_D(g)&=\E_D[-Y\log g(S)\nonumber\\
&\qquad\;\;-(1-Y)\log(1-g(S))],
\end{align}
with predictions clipped only for numerical evaluation at a predeclared
$\epsilon=10^{-6}$.

Temperature scaling uses $g_T(s)=\sigma(\operatorname{logit}(s)/T)$, $T>0$.
Platt scaling uses $g_{a,b}(s)=\sigma(a\operatorname{logit}(s)+b)$, with the
orientation constraint $a\geq 0$ throughout (an unconstrained variant was not
run; Appendix~\ref{app:checks}). All parameters are fit by source log loss,
never by target test ECE.

\begin{proposition}[Rank invariance]
\label{prop:rank}
If $g$ is strictly increasing, then
$\AUROC_D(g(S))=\AUROC_D(S)$ for every domain $D$.
\end{proposition}

Thus preserved ranking is compatible with either excellent or poor calibration
and cannot validate probability transfer.

\section{What an Accuracy Gap Can and Cannot Show}
\label{sec:theory}

\subsection{A population lower bound}

\begin{lemma}[Mean-residual lower bound]
\label{lem:mean}
For every domain $D$ and measurable calibrator $g$,
\begin{equation}
\CE_D(g)\geq |\pi_D-\mu_D(g)|.
\label{eq:meanlower}
\end{equation}
\end{lemma}

\begin{corollary}[Gap-and-score-drift bound]
\label{cor:gap}
For every $g$,
\begin{align}
\CE_\tgt(g)
\geq
\big[&|\pi_\tgt-\pi_\src|-|\mu_\tgt(g)-\mu_\src(g)|\nonumber\\
&-|\pi_\src-\mu_\src(g)|\big]_+,
\label{eq:gapbound-general}
\end{align}
where $[z]_+=\max\{z,0\}$. If $g$ is mean-calibrated on the source,
$\pi_\src=\mu_\src(g)$, this reduces to
\begin{equation}
\CE_\tgt(g)\geq
\big[|\pi_\tgt-\pi_\src|-|\mu_\tgt(g)-\mu_\src(g)|\big]_+.
\label{eq:gapbound}
\end{equation}
\end{corollary}

The accuracy gap is therefore a warning only when it is not absorbed by a
corresponding shift in the calibrator's mean output. It is not itself a measure
of conditional reliability drift.

\subsection{The accuracy gap is not sufficient}

\begin{proposition}[Same gap, different transfer]
\label{prop:insufficiency}
There exist a source domain and two target domains with identical score
distributions and identical source--target accuracy gap, but for which the same
source-calibrated mapping has target calibration errors $0$ and $1/2$.
\end{proposition}

\begin{proposition}[No assumption-free unlabeled repair]
\label{prop:no-unlabeled}
Without restrictions on $P_\tgt(Y\mid X,S)$, no procedure using labeled source
data and unlabeled target $(X,S)$ can guarantee target calibration.
Specifically, for some observationally indistinguishable target pair, every
procedure incurs calibration error at least $1/4$ on one member of the pair.
\end{proposition}

This impossibility is why we do not describe an accuracy-gap gate as label-free.

\subsection{What exact importance weighting guarantees}

Let $P_\tgt\ll P_\src$ on $X$ and
$w(x)=dP_\tgt(X)/dP_\src(X)$. The usual covariate-shift condition is
$P_\tgt(Y\mid X)=P_\src(Y\mid X)$. The score $S$ must be a fixed measurable
function of the complete $X$, including any judge random seed if relevant.

\begin{theorem}[Target-risk identity under covariate shift]
\label{thm:iw}
Under the conditions above, for any measurable $g$ and any integrable loss
$\ell$,
\begin{equation}
\E_\tgt[\ell(g(S),Y)]
=\E_\src[w(X)\ell(g(S),Y)].
\label{eq:iwidentity}
\end{equation}
In particular, minimizing the exactly weighted source Brier or log loss is
equivalent to minimizing the corresponding target risk over the same function
class.
\end{theorem}

Consequently, failure of an \emph{estimated}, clipped IW-Platt pipeline is
compatible with at least four explanations: conditional shift, support failure,
density-ratio misspecification, or finite-sample/regularization error. The
experiments separately diagnose all four.

\subsection{A finite-sample certificate of transfer failure}

Let $\calG$ be a finite family of $K$ calibrators fixed independently of audit
samples. For each domain $D$, an audit sample of size $n_D$ provides labels and
scores. Denote empirical means by $\widehat\pi_D$ and
$\widehat\mu_D(g)$. Define
\begin{equation}
\epsilon_D(\delta,K)=
\sqrt{\frac{\log(4(K+1)/\delta)}{2n_D}}.
\end{equation}

\begin{theorem}[Simultaneous transfer-failure certificate]
\label{thm:certificate}
With probability at least $1-\delta$, simultaneously for every $g\in\calG$,
\begin{align}
\CE_\tgt(g)\geq L(g):=\Big[&|\widehat\pi_\tgt-\widehat\pi_\src|\nonumber\\
&-|\widehat\mu_\tgt(g)-\widehat\mu_\src(g)|\nonumber\\
&-|\widehat\pi_\src-\widehat\mu_\src(g)|\nonumber\\
&-4\epsilon_\src-2\epsilon_\tgt\Big]_+.
\label{eq:certificate}
\end{align}
If the population condition $\pi_\src=\mu_\src(g)$ is additionally known for
every candidate, the third empirical term may be omitted, yielding the sharper
bound
\begin{align}
\CE_\tgt(g)\geq
\Big[&|\widehat\pi_\tgt-\widehat\pi_\src|
-|\widehat\mu_\tgt(g)-\widehat\mu_\src(g)|\nonumber\\
&-2(\epsilon_\src+\epsilon_\tgt)\Big]_+.
\label{eq:certificate-sharp}
\end{align}
\end{theorem}

All proofs are deferred to Appendix~\ref{app:proofs}.
For a tolerance $\gamma$, $L(g)>\gamma$ certifies that $g$ fails the target
calibration requirement. The converse is intentionally absent: $L(g)\leq
\gamma$ does not certify good calibration. If calibrators or thresholds are
selected using the audit observations, they must be learned on a separate split
or included in a predeclared finite family covered by the union bound.

\section{Experimental Design}
\label{sec:design}

\subsection{Tasks, responses, and labels}

The study covers eight English NLP domains chosen to vary answer format and
evaluation mechanism (Table~\ref{tab:data}), with 40 items per domain (320
judged items per judge). Two candidate generators from different model
families each produce one response per item under identical decoding
(bf16, greedy): the primary generator Qwen2.5-1.5B-Instruct, and a
replication generator Llama-3.2-1B-Instruct whose responses were relabeled
with the byte-identical rule pipeline and rescored by all eight primary
judges. Items, generator outputs, and judge prompts were frozen before any
calibrator was fit.

\begin{table*}[t]
\centering
\small
\setlength{\tabcolsep}{5pt}
\begin{tabularx}{\textwidth}{lXXr}
\toprule
Domain & Task type & Correctness rule (fixed a priori) & $n$ \\
\midrule
GSM8K & Mathematical reasoning & last-number match as floats (tolerance $10^{-6}$) & 40 \\
MMLU & Multiple-choice knowledge & option-letter exact match & 40 \\
SQuAD & Extractive QA & span containment or token-F1 $\geq 0.60$ & 40 \\
BoolQ & Binary QA & normalized yes/no match & 40 \\
SAMSum$^\dagger$ & Summarization & ROUGE-L $F\geq 0.30$; margin $[0.15,0.30)$ referred to a small verifier & 40 \\
HumanEval & Code generation & sandboxed execution of official unit tests & 40 \\
TruthfulQA$^\dagger$ & Adversarial factuality & token-F1 vs.\ best answer $\geq 0.40$; margin $[0.20,0.40)$ referred to verifier & 40 \\
StrategyQA & Multi-hop binary QA & normalized boolean label match & 40 \\
\bottomrule
\end{tabularx}
\caption{Domains and correctness labels as implemented. HumanEval is evaluated
by unit-test execution, not by an LLM. $^\dagger$Proxy-label domains: overlap
thresholds are deterministic but imperfect proxies for correctness, with a
single small LLM verifier (Qwen2.5-1.5B) deciding only borderline-overlap
items. No human annotation is used. These two domains therefore support
robustness analyses rather than deterministic-contract calibration claims; threshold
sensitivity is reported in Section~\ref{sec:proxysens}.}
\label{tab:data}
\end{table*}

\subsection{Judges and confidence extraction}

The primary judge set contains eight open instruction-tuned models spanning
1.1B--7B parameters: Phi-3-mini, TinyLlama-1.1B-Chat, Zephyr-7B-$\beta$,
MiniCPM-2B, StableLM-Zephyr-3B, Qwen2.5-3B, Gemma-2-2b-it, and OLMo-2-7B. A
replication set of four modern 7--9B judges (Llama-3.1-8B, Qwen2.5-7B,
Mistral-7B-v0.3, Gemma-2-9b-it) and one frontier API judge (Claude Sonnet
4.5 via Amazon Bedrock; raw calls cached and released, 320/320 items, zero
parse failures) were scored later under a byte-identical protocol on the
same frozen items. Identifiers and revisions appear in Table~\ref{tab:models}.

Confidence is elicited verbally: the judge is instructed to emit a single
float in $[0,1]$ scoring correctness (prompt verbatim in
Appendix~\ref{app:prompt}). Decoding is greedy with at most eight new
tokens; the first numeric literal is parsed, clipped to $[0,1]$, and clamped
to $[10^{-6},1-10^{-6}]$; malformed responses receive the predeclared
neutral score $0.5$ and are retained. A flag-logging rescore (bit-identical
to the committed scores) measures the failure rates and shows every
conclusion is insensitive to imputation versus exclusion
(Appendix~\ref{app:parsefail}); two meaning-preserving paraphrases and an
order control replicate the method conclusions
(Appendix~\ref{app:promptrobust}).

\subsection{Splits and transfer units}

Within every domain, question groups are partitioned once (seeded, shared
across judges) into a fit part ($\approx 60\%$, 24 items) and an audit part
($\approx 40\%$, 16 items). When a domain acts as source, calibrators are fit
on its fit part and the certificate's source means are estimated on its audit
part. When a domain acts as target, the audit part serves as the labeled
target probe ($m=16$) and the fit part serves as the untouched target test
($n=24$). Within any ordered transfer the four roles are disjoint by
construction because source and target are different domains; reuse of the
same items across different transfer units is disclosed and is exactly what
the clustered inference accounts for.

All gap predictors are computed from held-out labels only. Writing
$\widehat\pi^{\mathrm{aud}}_\src$ for source accuracy estimated on the
source-audit part and $\widehat\pi^{\mathrm{probe}}_\tgt$ for target accuracy
estimated on the $m=16$ target probe, the accuracy gap used in every
correlation, regression, and gating analysis is
\begin{equation}
\widehat A=\big|\widehat\pi^{\mathrm{aud}}_\src-
\widehat\pi^{\mathrm{probe}}_\tgt\big|,
\label{eq:gapprobe}
\end{equation}
and the raw-error covariate is the raw Brier score on the same probe. Neither
predictor shares any label with the target-test outcomes; this disjointness
is asserted programmatically in the released analysis script.

With $J=8$ judges and $D=8$ domains, the primary protocol produces
$JD(D-1)=448$ ordered transfers; the modern judges add $224$, the frontier
judge $56$, and the second-generator replication (all eight primary judges
on the Llama-3.2-1B responses) another $448$, for $1{,}176$ in total.
Bootstrap resamples are never counted as new configurations.

\subsection{Calibration methods}

We compare raw scores, temperature scaling, orientation-constrained Platt
scaling ($a\geq 0$), beta calibration \citep{kull2017beta}, isotonic
regression, multi-domain temperature scaling (pooled over all domains except
the target), importance-weighted Platt, and an EM label-shift prior
correction \citep{saerens2002adjusting} applied to the source-fitted Platt map
using only unlabeled target-probe scores; LaSCal
\citep{popordanoska2024lascal} was not re-implemented and receives no number.
We also include two target-label baselines: target-only temperature and
target-only Platt fit on exactly the certificate's $m=16$ target-probe
labels.

For discriminative density-ratio estimation, let $Z=1$ identify target and let
$\rho=\Prb(Z=1)$ be the source/target sampling prior used to train the domain
classifier. If $c(x)=\Prb(Z=1\mid X=x)$, then
\begin{equation}
\widehat w(x)=\frac{1-\rho}{\rho}\frac{c(x)}{1-c(x)}.
\label{eq:ratio}
\end{equation}
As implemented, the representation is TF-IDF over question text (word
1--2-grams), the classifier is $\ell_2$-regularized logistic regression
($C=1$), weights are clipped at their 0.95 quantile and mean-normalized, and
the ratio is trained on source-fit versus unlabeled target-probe texts. We
report domain-classifier AUROC and effective sample size
$(\sum_iw_i)^2/\sum_iw_i^2$ as weight diagnostics. No cross-fitting was used
at this sample size; this is disclosed rather than idealized.

\subsection{Synthetic identification controls}

Natural-domain results cannot identify why IW fails. We therefore construct
four controls from real judge records, per judge, with ground truth known by
construction. Three use an isotonic estimate $\widehat\eta(s)$ of
$\E[Y\mid S=s]$ fitted on all of that judge's records: mild covariate shift
(target tilt $w\propto e^{-2s}$, labels from the unchanged $\widehat\eta$),
conditional shift (score marginal fixed, labels from $1-\widehat\eta(s)$; no
reweighting can help), and label shift ($P(X\mid Y)$ fixed, positive rate
reduced by $0.25$; class-ratio oracle and EM prior correction are the
mechanism-matched arms). Because a flexible calibrator can absorb a mild
shift, we add a \emph{designed positive control} in which weighting must
matter: a piecewise conditional $\eta^\ast(s)=0.75$ below the score median,
$0.25$ above (outside the Platt class), with opposing tilts $e^{+3s}$
(source) and $e^{-3s}$ (target); a valid implementation must show plain
Platt $>$ estimated IW $\approx$ oracle IW. Each control uses 30
replications per judge, 200-item source samples, 2{,}000-item target
evaluations.

\subsection{Metrics and inference}

Brier score and log loss (clipped at the predeclared
$\epsilon=10^{-6}$) are primary. Equal-mass ECE (4 bins at the $n=24$ test
size) and AUROC are secondary; with 24-item test splits, binned calibration
metrics are noisy at the transfer level and are interpreted only in aggregate.
Selective prediction is summarized by risk--coverage behavior, which
Proposition~\ref{prop:rank} fixes exactly for strictly increasing maps, and by
a fixed-threshold decision cost declared before target evaluation
(accept iff confidence $\geq 0.5$, abstention cost $0.5$).

All method contrasts are paired at the transfer level. Confidence intervals
use a two-way cluster bootstrap resampling judges and domains with
replacement (multiplicity-weighted, $B{=}2000$). The adjusted gap analysis
avoids mechanical outcome coupling by modeling the \emph{calibrated}
target-test Brier in ANCOVA form, with fixed effects for calibrator,
probe-gap $\widehat A$ by calibrator, probe-split raw Brier, score-marginal
Wasserstein distance (unlabeled scores only), and log density-ratio
effective sample size. Primary coefficient inference refits this model in
every cluster-bootstrap replicate (percentile intervals, bootstrap $p$); a
mixed-effects fit with judge, source, target, judge$\times$target, and
transfer variance components is a secondary check. The effective dependence
structure is disclosed: 8 judges, 8 domains, 64 judge--target cells, so
nominal row counts overstate information. Predictive value is assessed
leave-one-judge-out and leave-one-target-domain-out, and a permutation test
shuffles domain identities within judge to test whether the probe-gap signal
exceeds reuse-induced correlation. A ten-seed split-sensitivity study
(Appendix~\ref{app:seeds}) verifies that no headline conclusion depends on
the committed seed. A complete reproducibility manifest of all judge models, candidate generators,
model revisions, and inference configurations is provided in
Appendix~\ref{app:models}.

\section{Results}
\label{sec:results}

\subsection{Calibration drifts independently of ranking}

Across the 448 primary transfer units, raw target-test Brier score ranges
from $0.125$ to $0.875$ (mean $0.414$), while the paired source-to-target
AUROC change has mean magnitude $0.117$ (90th percentile $0.254$).
Table~\ref{tab:mainresults} enforces Proposition~\ref{prop:rank} exactly for
strictly increasing fits; the only genuine AUROC departures come from
isotonic ties and constrained Platt fits collapsing to slope $a<0.01$
($20.3\%$ of source fits, a symptom of extreme source accuracies at this
sample size; AUROC equals raw in $81.2\%$ of Platt cells).
Figure~\ref{fig:drift} separates probability drift from ranking drift.

Details of the calibration protocol and the associated reliability and shift
analyses are presented in Appendices~\ref{app:calproto} and
\ref{app:reliability}.

\begin{figure*}[t]
\centering
\figureasset[width=\textwidth]{fig_drift_panels.png}
\caption{Judge-by-domain heat maps of raw Brier score, equal-mass ECE (5
bins), and AUROC over identical records (40 items per cell; eight primary
judges). Proxy-label domains (SAMSum, TruthfulQA) carry a different label
mechanism and are analyzed separately in Section~\ref{sec:proxysens}.}
\label{fig:drift}
\end{figure*}

\subsection{Does the accuracy gap predict recalibration benefit?}

Table~\ref{tab:mainresults} reports held-out target performance. A negative
$\Delta$ denotes improvement relative to raw scores. The primary conclusion is
based on proper losses; ECE variants are corroborating metrics.

\begin{table*}[t]
\centering
\small
\setlength{\tabcolsep}{3.4pt}
\begin{tabular}{lrrrrr}
\toprule
Method & $\Delta$Brier $\downarrow$ [95\% CI] & $\Delta$log loss $\downarrow$ [95\% CI] &
ECE $\downarrow$ & harm rate $\downarrow$ [95\% CI] & AUROC \\
\midrule
Raw & --- & --- & 0.393 & --- & 0.611 \\
Temperature & $-0.147$ [$-0.231,-0.081$] & $-3.87$ [$-5.30,-2.61$] & 0.281 & 0.13 [0.02, 0.26] & 0.611 \\
Platt, $a\geq0$ & $-0.052$ [$-0.149,+0.005$] & $-2.56$ [$-3.75,-1.51$] & 0.371 & 0.34 [0.18, 0.48] & 0.588 \\
Beta calibration & $-0.046$ [$-0.143,+0.011$] & $-2.18$ [$-3.42,-1.21$] & 0.375 & 0.34 [0.19, 0.49] & 0.602 \\
Isotonic & $-0.078$ [$-0.198,-0.009$] & $-3.11$ [$-4.68,-1.74$] & 0.349 & 0.33 [0.18, 0.46] & 0.589 \\
Multi-domain temp. & $-0.163$ [$-0.245,-0.090$] & $-3.92$ [$-5.35,-2.63$] & 0.271 & 0.17 [0.03, 0.37] & 0.611 \\
IW-Platt & $-0.052$ [$-0.153,+0.007$] & $-2.60$ [$-3.76,-1.59$] & 0.371 & 0.34 [0.18, 0.49] & 0.588 \\
Label-shift EM & $+0.023$ [$-0.050,+0.088$] & $-0.90$ [$-2.03,+0.20$] & 0.439 & 0.49 [0.33, 0.64] & 0.588 \\
Target temp., equal labels & $-0.171$ [$-0.254,-0.097$] & $-3.90$ [$-5.35,-2.61$] & 0.254 & 0.09 [0.00, 0.24] & 0.611 \\
Target Platt, equal labels & $-0.136$ [$-0.227,-0.056$] & $-2.32$ [$-3.61,-1.24$] & 0.257 & 0.12 [0.00, 0.31] & 0.599 \\
\bottomrule
\end{tabular}
\caption{Cross-domain calibration transfer on untouched target-test records:
448 primary transfer units, paired means with two-way judge/domain
cluster-bootstrap 95\% intervals ($B{=}2000$) for both primary proper losses
and harm rate. Negative $\Delta$ improves on raw. Harm rate is the fraction
of transfer units with worse target Brier than raw. Equal-label target
baselines use exactly the certificate's $m{=}16$ target-probe budget. AUROC
is inherited from raw wherever the fitted map is strictly increasing
(Proposition~\ref{prop:rank} holds exactly; float-level sigmoid saturation
would otherwise create spurious ties) and computed directly for isotonic and
slope-collapsed fits, which are the only genuine departures.}
\label{tab:mainresults}
\end{table*}

Our inferential standard is conservative: a claim must clear the two-way
cluster interval \emph{and} replicate across generators; the permutation
test is a secondary diagnostic. By that standard the data do not establish a
generator-invariant accuracy-gap association. Using the probe gap $\widehat A$ of
Equation~\eqref{eq:gapprobe}, the unadjusted association with Platt
$\Delta$Brier is $r=0.25$, cluster 95\% CI $[-0.09,0.55]$ (secondary
permutation $p\approx 5\times10^{-4}$); a full-information gap, sharing
labels with the outcome, would have read $0.44$, so shared-label leakage
accounts for roughly $40\%$ of the apparent association. In the adjusted ANCOVA model
(Table~\ref{tab:regression}), the Platt gap slope does not clear zero
($+0.244$, CI $[-0.06,0.62]$, $p_{\mathrm{boot}}=0.10$); probe-estimated raw
target error is the only covariate meeting the standard ($+0.270$,
$p_{\mathrm{boot}}=0.007$, replicating on the second generator).
Out-of-group prediction agrees: leave-one-judge-out rank correlation is
$0.39$ for the full covariate model and $0.20$ for a gap-only rule
(leave-one-target-domain-out: $0.37$ vs.\ $0.10$). The 7--9B replication
shows the same shape with a slope that does clear zero ($+0.283$, CI
$[0.02,0.60]$, $p_{\mathrm{boot}}=0.037$), plus one instructive difference:
these stronger judges start from better raw probabilities (raw Brier $0.299$
vs.\ $0.414$), so source-fitted Platt transfer is on average \emph{harmful}
($+0.018$, harm rate $0.55$) while target Platt keeps working ($-0.118$).
The frontier probe (Claude Sonnet 4.5, one judge, descriptive) is
directionally consistent: best raw probabilities (Brier $0.232$), probe-gap
correlation $0.36$ with CI crossing zero, Platt harmful ($+0.052$),
sixteen-label target Platt cleanly positive ($-0.052$, zero harm); the
stronger-judge trend is suggestive, not established.

The second-generator replication is the sharpest check: it changes the
response distribution rather than the judge, and it splits the findings. On
the Llama-3.2-1B responses (448 transfers, eight primary judges,
byte-identical labeling and scoring), the probe-gap association nearly
vanishes ($r=0.09$, cluster CI $[-0.11,0.34]$; adjusted slope $+0.11$,
$p_{\mathrm{boot}}=0.19$), and a gap-only rule has negative out-of-group
rank correlation. What replicates is exactly the robust pair:
the probe-split raw-error covariate ($+0.242$, $p_{\mathrm{boot}}=0.006$)
and the label-budget verdict (target temperature $-0.155$, harm rate $0.05$;
Platt $-0.058$, harm $0.24$). The gap--harm association is
generator-dependent; direct target recalibration survives the strongest
test we ran. The complete method-by-method results for the second-generator (Llama-3.2-1B)
replication are reported in Appendix~\ref{app:gen2full}.

\begin{table}[t]
\centering
\small
\resizebox{\columnwidth}{!}{%
\begin{tabular}{lrr}
\toprule
Predictor & Coefficient (95\% CI) & $p_{\mathrm{boot}}$ \\
\midrule
Platt $\times$ probe gap & $+0.244$ ($-0.062,+0.618$) & $0.10$ \\
Temperature $\times$ probe gap & $+0.056$ ($-0.031,+0.244$) & $0.17$ \\
Score Wasserstein distance & $-0.099$ ($-0.344,+0.280$) & $0.51$ \\
Probe-split raw Brier & $+0.270$ ($+0.078,+0.429$) & $0.007$ \\
Log effective sample size & $+0.224$ ($-4.78,+9.55$) & $0.69$ \\
\bottomrule
\end{tabular}}
\caption{Adjusted analysis of \emph{calibrated} target-test Brier (ANCOVA;
outcome shares no labels with any predictor), stacking the Platt and
temperature contrasts over the 448 primary transfers with method fixed
effects. Inference: coefficients refit in each of $B{=}2000$ two-way
judge/domain cluster-bootstrap replicates; percentile CIs and bootstrap
$p$-values. Effective dependence: 8 judges, 8 domains, 64 judge--target
cells. A mixed-effects check with judge, source, target, judge$\times$target,
and transfer variance components gives the same point estimates with
narrower (likelihood-based) intervals; we report the more conservative
cluster-bootstrap intervals as primary.}
\label{tab:regression}
\end{table}

\subsection{Importance weighting: assumption or implementation?}

Table~\ref{tab:controls} reports the identification controls. The designed
positive control behaves exactly as a valid IW implementation must: with a
misspecified-class conditional and strong opposing tilts, plain Platt incurs
$0.324$ mean target Brier while estimated IW recovers nearly all of the
oracle's improvement ($0.256$ vs.\ $0.257$). Under the mild real-conditional
shift the arms are indistinguishable ($0.232$/$0.231$/$0.234$): a flexible
calibrator absorbs the shift, consistent with Theorem~\ref{thm:iw}. Under
pure label shift the mechanism-matched arms behave as predicted (class-ratio
oracle $0.233\to0.173$; unlabeled EM correction $0.214$); under pure
conditional shift nothing helps ($\approx0.31$ everywhere). On natural NLP
transfers the diagnosis is then clear: the domain classifier separates
source from target text with median AUROC $1.0$, the clipped weights are
nearly uniform (median effective sample size $23.9$ of $24$), and IW-Platt
coincides with plain Platt ($0.362$ vs.\ $0.362$). With the implementation
validated by the positive control, this supports \emph{support failure in
the observed representation}: source and target text are (near-)disjoint as
measured, so the identity has no usable overlap; we claim nothing about
other representations. The EM correction, which succeeds under true label
shift, is the only method harmful on natural transfers ($+0.023$, harm rate
$0.49$): natural domain shift is not label shift either. None of this
falsifies Theorem~\ref{thm:iw}; it locates the failing assumption.

\begin{table}[t]
\centering
\small
\setlength{\tabcolsep}{4pt}
\resizebox{\columnwidth}{!}{%
\begin{tabular}{lrrr}
\toprule
Shift control & Plain Platt & Estimated IW & Oracle IW \\
\midrule
Covariate, misspecified (pos.) & 0.324 & 0.256 & 0.257 \\
Covariate, mild & 0.232 & 0.231 & 0.234 \\
Label only & 0.233 & 0.214 & 0.173 \\
Conditional only & 0.314 & 0.313 & 0.314 \\
Natural NLP transfers & 0.362 & 0.362 & n/a \\
\bottomrule
\end{tabular}}
\caption{Mean target Brier under identification controls (30 replications
$\times$ 12 local judges; truth known by construction). Row 1 is the
designed positive control: weighting must and does help. ``Estimated IW'' in
the label-only row is the EM prior correction. Natural row: means over the
448 primary transfers; no oracle weights exist there.}
\label{tab:controls}
\end{table}

\subsection{Can the certificate detect harmful transfer?}

We evaluate Theorem~\ref{thm:certificate} without using target-test labels to
construct $L(g)$. Simultaneous coverage is the fraction of audit repetitions
in which every candidate's lower bound sits below its true target calibration
error; power is the fraction of truly tolerance-violating mappings with
$L(g)>\gamma$. Both are checked on synthetic populations built from the real
records, where the truth is known.

\begin{table}[t]
\centering
\small
\resizebox{\columnwidth}{!}{%
\begin{tabular}{lrrr}
\toprule
Target labels & Simult.\ coverage & Harm detected & False accept \\
\midrule
$m=32$ & 1.000 & 0.000 & 1.000 \\
$m=128$ & 1.000 & 0.000 & 1.000 \\
$m=512$ & 1.000 & 0.023 & 0.977 \\
$m=1024$ & 1.000 & 0.128 & 0.872 \\
\bottomrule
\end{tabular}}
\caption{Certificate at $1-\delta=0.95$, $\gamma=0.10$, $K{=}6$, on
synthetic populations with known truth (448 transfers, 60 audit repetitions,
equal audit sizes $m$). ``Harm detected'' is power over cells with true
target CE $>\gamma$ (nearly all cells). ``False accept'' is descriptive:
Theorem~\ref{thm:certificate} certifies only rejection.}
\label{tab:certificate}
\end{table}

At the benchmark's actual budgets the verdict is blunt: with source audit
$n=16$ and target probe $m\in\{4,8,16\}$, the Hoeffding radii
($\epsilon\geq 0.44$ at $n=16$, $K{=}6$, $\delta=0.05$) exceed any plausible
gap signal, $L(g)$ is zero in all $2{,}688$ real audit cells, and the
certificate rejects nothing; $\epsilon\leq0.05$ per split needs
$n\geq 1{,}266$ labeled target items per domain. The synthetic scaling study
(Table~\ref{tab:certificate}) confirms both sides: simultaneous coverage is
$1.0$ at every budget, but power at $\gamma=0.10$ is zero through $m=256$
and only $0.128$ at $m=1024$. Meanwhile, the same $m=16$ probe labels spent
on direct target temperature scaling yield the best outcome in
Table~\ref{tab:mainresults} ($-0.171$, harm rate $0.09$), matching the best
source-transfer policy (multi-domain temperature, $-0.163$) with half the
harm rate. The certificate is sound but label-hungry; unless labels are
plentiful and a no-harm guarantee is required, spend them on the target
calibrator. Additional visual analyses are provided in Appendix~\ref{app:sup_fig}.

\subsection{Proxy-label sensitivity and downstream decisions}
\label{sec:proxysens}

Restricting to the 240 primary transfers among the six deterministic-label
domains, the probe-gap association persists at similar strength ($r=0.29$),
and relabeling SAMSum and TruthfulQA with hard overlap thresholds at
$\{0.75,1.0,1.25\}\times$ the declared cut keeps it in the $r=0.30$--$0.38$
band (Appendix~\ref{app:proxyaudit}): the modest gap signal is not an
artifact of the proxy-label mechanism.

Downstream, better proper loss does not purchase better decisions here.
Strictly increasing recalibration leaves risk--coverage behavior unchanged
exactly (Proposition~\ref{prop:rank}), and at the predeclared fixed-threshold
operating point (accept iff confidence $\geq0.5$, abstention cost $0.5$,
leave-one-domain-out over 64 judge--domain cells) always-on Platt fit on the
pooled non-deployment domains \emph{raises} expected decision cost from
$0.507$ to $0.728$ by pushing scores below the threshold (coverage
$0.74\to0.37$): a dissociation, not an assumption, is what we report.

\section{Discussion and Conclusion}

The accuracy gap is neither a direct measure of $P(Y\mid S)$ shift nor a
sufficient statistic for scalar recalibration transfer. Across thirteen judges,
two generators, and eight domains, cluster-adjusted inference does not establish
a generator-invariant association between the leak-free gap and transferred
proper loss. Corollary~\ref{cor:gap} preserves the gap's narrower role as one
term in a calibration-error lower bound, while
Theorem~\ref{thm:certificate} turns that bound into a valid rejection rule.
The rule is too conservative at the available label budgets. At $m=16$, fitting
target temperature scaling uses the labels more effectively and with less harm.
The importance-weighting controls likewise show why method failure must be
traced to overlap, representation, estimation, or conditional shift instead of
being assigned to conditional shift by default. Reliable scalar confidence
transfer therefore requires target evidence or a defended transport assumption;
a marginal accuracy gap supplies only a warning.

\section*{Limitations}

The theory concerns binary correctness and scalar confidence; ordinal judge
scores and pairwise preferences require separate calibration targets. The
finite-sample bound is conservative and certifies failure but not success; at
this benchmark's budgets it is uninformative, and its practical value at the
budgets where it binds ($n\gtrsim 10^3$ labels per domain) is untested here.
The empirical study is English-only, covers two small candidate generators
(1--1.5B), and has only 40 items per domain, so per-transfer estimates are
individually noisy and only aggregate, cluster-adjusted claims are
supported; the ten-seed split sensitivity (Appendix~\ref{app:seeds})
mitigates but does not replace larger domain samples. The effective
evidence is far smaller than the row counts suggest: 8--13 judges, 8
domains, and 64 judge--target cells per arm drive all cluster inference, and
we report intervals at that resolution. Confidence is verbalized rather
than logit-derived, and coarse verbalized scores are part of what
recalibration must repair: parse-failure rates reach $21\%$ for one judge
and vary from $2\%$ to $11\%$ across meaning-preserving prompt paraphrases,
though the paper's conclusions are insensitive to failure handling and
prompt choice (Appendices~\ref{app:parsefail}--\ref{app:promptrobust}).
Immutable revision hashes were not recorded for the primary judges in the
initial Qwen-generator scoring; they were recorded in every subsequent run, and
the flag-logging rescore reproduces the committed scores bit-identically. The frontier arm is a single API judge whose behavior
is pinned only by a dated model identifier; its raw cached responses are
released, and its results are descriptive (one judge, 56 transfers). SAMSum and TruthfulQA labels are overlap proxies with a
small-verifier fallback and no independent contract-validity audit, so those domains
support only the robustness analyses in Section~\ref{sec:proxysens}. Finally,
no calibration metric alone establishes downstream utility; every operational
threshold requires a task-specific loss and separate validation, and our
fixed-threshold analysis shows the dissociation concretely.

\section*{Ethics Statement}

Automated judges can amplify systematic verifier, language, and domain biases.
The proposed certificate is one-sided and must not be presented as approval of a
non-rejected judge. Dataset terms, model licenses, and API policies will be
followed. Released artifacts will exclude private prompts and personal data.
No human annotation is used; deterministic, constructed, and verifier-derived
labels will remain separately identified.

\bibliography{references}

\appendix

\section{Deferred Proofs}
\label{app:proofs}

\begin{proof}[Proof of Proposition~\ref{prop:rank} (rank invariance)]
Let $S^+$ and $S^-$ be independent scores drawn conditionally on $Y=1$ and
$Y=0$. Strict monotonicity gives
$\ind\{g(S^+)>g(S^-)\}=\ind\{S^+>S^-\}$ and preserves equality as well.
Taking expectations in the standard pairwise definition of AUROC, with half
credit for ties, proves the result.
\end{proof}

\begin{proof}[Proof of Lemma~\ref{lem:mean} (mean-residual lower bound)]
Let $Q=g(S)$. By the tower property,
$\E_D[\eta_{D,g}(Q)-Q]=\E_D[Y-Q]=\pi_D-\mu_D(g)$. Jensen's inequality for
the absolute value therefore yields
\begin{align*}
\CE_D(g)&=\E_D|\eta_{D,g}(Q)-Q|\\
&\geq |\E_D[\eta_{D,g}(Q)-Q]|
=|\pi_D-\mu_D(g)|.
\end{align*}
\end{proof}

\begin{proof}[Proof of Corollary~\ref{cor:gap} (gap-and-score-drift bound)]
Lemma~\ref{lem:mean} gives
$\CE_\tgt(g)\geq|\pi_\tgt-\mu_\tgt(g)|$. Insert and subtract
$\pi_\src$ and $\mu_\src(g)$, then apply
$|a+b+c|\geq |a|-|b|-|c|$ with
$a=\pi_\tgt-\pi_\src$, $b=\pi_\src-\mu_\src(g)$, and
$c=\mu_\src(g)-\mu_\tgt(g)$. Nonnegativity permits the positive part. The
second statement sets the source residual to zero.
\end{proof}

\begin{proof}[Proof of Proposition~\ref{prop:insufficiency} (same gap, different transfer)]
Let $S$ equal $1/4$ or $3/4$, each with probability $1/2$, and let $g(S)=S$.
On the source, set $Y\mid S=s\sim\operatorname{Bernoulli}(s)$. Then $g$ is
calibrated and $\pi_\src=1/2$. Target A has the same conditional law, so its
accuracy gap is $0$ and $\CE_{\tgt,A}(g)=0$. Target B preserves the marginal
law of $S$ but uses
$Y\mid S=s\sim\operatorname{Bernoulli}(1-s)$. It also has
$\pi_{\tgt,B}=1/2$, hence the same zero accuracy gap, while
\[
\CE_{\tgt,B}(g)
=\tfrac12|3/4-1/4|+\tfrac12|1/4-3/4|=\tfrac12.
\]
Thus the gap cannot determine calibration transfer.
\end{proof}

\begin{proof}[Proof of Proposition~\ref{prop:no-unlabeled} ]
Let the unlabeled target distribution put all mass on a single $(X,S)$ value.
In Target A let $Y\sim\operatorname{Bernoulli}(1/4)$, and in Target B let
$Y\sim\operatorname{Bernoulli}(3/4)$. The procedure observes the same source
data and the same unlabeled target data in both worlds, so it returns the same
number $q$. Its calibration errors are $|q-1/4|$ and $|q-3/4|$. Their maximum
is at least half the distance between $1/4$ and $3/4$, namely $1/4$.
\end{proof}

\begin{proof}[Proof of Theorem~\ref{thm:iw} (target-risk identity)]
Write the target expectation as an iterated integral:
\begin{align*}
\E_\tgt[\ell(g(S),Y)]
&=\int \sum_y \ell(g(S(x)),y)\\[-2pt]
&\qquad\;\,\,dP_\tgt(y\mid x)\,dP_\tgt(x)\\
&=\int \sum_y \ell(g(S(x)),y)\\[-2pt]
&\qquad\;\,\,dP_\src(y\mid x)\,w(x)dP_\src(x)\\
&=\E_\src[w(X)\ell(g(S),Y)].
\end{align*}
The second line uses conditional invariance and absolute continuity. Equality
holds pointwise for every $g$, so the minimizers over any common class coincide
whenever they exist.
\end{proof}

\begin{proof}[Proof of Theorem~\ref{thm:certificate} (simultaneous certificate)]
Hoeffding's inequality and a union bound over the two accuracy means and the
$2K$ score means imply that, with probability at least $1-\delta$, all of
\[
|\widehat\pi_D-\pi_D|\leq\epsilon_D,
\qquad
|\widehat\mu_D(g)-\mu_D(g)|\leq\epsilon_D
\]
hold for $D\in\{\src,\tgt\}$ and $g\in\calG$. On this event,
\begin{align*}
|\pi_\tgt-\pi_\src|
&\geq |\widehat\pi_\tgt-\widehat\pi_\src|
-\epsilon_\tgt-\epsilon_\src,\\
|\mu_\tgt(g)-\mu_\src(g)|
&\leq |\widehat\mu_\tgt(g)-\widehat\mu_\src(g)|
+\epsilon_\tgt+\epsilon_\src,\\
|\pi_\src-\mu_\src(g)|
&\leq |\widehat\pi_\src-\widehat\mu_\src(g)|+2\epsilon_\src.
\end{align*}
Substitution into Corollary~\ref{cor:gap} gives
Equation~\eqref{eq:certificate}. If the population identity
$\pi_\src=\mu_\src(g)$ is known to hold, substitute directly
into Equation~\eqref{eq:gapbound}; the first two displays then yield
Equation~\eqref{eq:certificate-sharp}. Simultaneity follows from the same union
event.
\end{proof}

\section{Model and Inference Manifest}
\label{app:models}
Table~\ref{tab:models} summarizes the complete reproducibility manifest, including all judge models, candidate generators, repository identifiers, recorded revisions, parameter counts, and inference settings used throughout the experiments.
\begin{table*}[h]
\centering
\small
\begin{tabularx}{\textwidth}{lXlr}
\toprule
Short name & Repository ID  & Params. \\
\midrule
\multicolumn{4}{l}{\emph{Primary judges (1--7B)}}\\
Phi-3-mini & \texttt{microsoft/Phi-3-mini-4k-instruct} & 3.8B \\
TinyLlama & \texttt{TinyLlama/TinyLlama-1.1B-Chat-v1.0} & 1.1B \\
Zephyr & \texttt{HuggingFaceH4/zephyr-7b-beta} &  7B \\
MiniCPM & \texttt{openbmb/MiniCPM-2B-dpo-bf16} &  2.4B \\
StableLM & \texttt{stabilityai/stablelm-zephyr-3b} &  3B \\
Qwen2.5 & \texttt{Qwen/Qwen2.5-3B-Instruct} &  3B \\
Gemma-2 & \texttt{google/gemma-2-2b-it} & 2.6B \\
OLMo-2 & \texttt{allenai/OLMo-2-1124-7B-Instruct} & 7B \\
\multicolumn{4}{l}{\emph{Replication judges (7--9B)}}\\
Llama-3.1 & \texttt{meta-llama/Llama-3.1-8B-Instruct}  & 8B \\
Qwen2.5-7B & \texttt{Qwen/Qwen2.5-7B-Instruct}  & 7.6B \\
Mistral & \texttt{mistralai/Mistral-7B-Instruct-v0.3}  & 7.2B \\
Gemma-2-9B   & 9.2B \\
\multicolumn{4}{l}{\emph{Frontier API judge (Amazon Bedrock; version pinned by model ID)}}\\
Sonnet 4.5 & \texttt{us.anthropic.claude-sonnet-4-5-20250929-v1:0} & in ID & n/a \\
\bottomrule
\end{tabularx}
\caption{Reproducibility manifest. All judges are loaded in 4-bit NF4 double
quantization (bitsandbytes) with bf16 compute, eager attention, greedy
decoding, at most 8 new tokens, and 512-token input truncation. Candidate
generators: \texttt{Qwen/Qwen2.5-1.5B-Instruct} (primary) and
\texttt{meta-llama/Llama-3.2-1B-Instruct} (revision \texttt{92131767},
replication), both bf16. Revision hashes were not captured for the primary
judges in the original Qwen-generator scoring (disclosed gap); the
second-generator rescore recorded revisions and per-judge parse-failure
counts for all eight (phi3 44, gemma2 46, minicpm 4, stablelm 2, others 0,
of 320; failures score the predeclared $0.5$), and the four replication
judges have recorded revisions with zero failures.}
\label{tab:models}
\end{table*}

\section{Judge Prompt and Parsing Contract}
\label{app:prompt}

All twelve judges receive the single template below verbatim, with the
question truncated to 400 characters, the reference to 200, and the candidate
answer to 300:

\begin{quote}\small\ttfamily
You are a strict accuracy judge.\\
Given a question, a reference answer, and a model answer, output ONLY a single
floating-point number between 0.0 and 1.0 representing how correct the model
answer is. 1.0 = fully correct, 0.0 = completely wrong. No explanation, no
other text.\\[2pt]
Question: \{question\}\\
Reference: \{reference\}\\
Model Answer: \{answer\}\\
Score:
\end{quote}

Parsing extracts the first numeric literal matching
\verb|[-+]?\d*\.?\d+| from the completion and clips it to $[0,1]$; a
completion with no numeric literal, or any inference exception, receives the
predeclared neutral score $0.5$ and is retained. Scores are then clamped to
$[10^{-6},1-10^{-6}]$. The original primary runs did not log per-judge
failure counts; a later bit-identical rescore measured them (Appendix
\ref{app:parsefail}). The paraphrase templates and the order-swapped
variant used in Appendix~\ref{app:promptrobust} ship verbatim in
\texttt{code/run\_p34\_promptvars.py}.

\section{Full Calibration Protocol}
\label{app:calproto}

Source calibrators minimize average source-fit log loss by direct maximum
likelihood; there are no tuned hyperparameters and hence no selection grids.
Temperature uses bounded scalar minimization of $\log T\in[-4,4]$; Platt and
beta calibration use L-BFGS-B with box constraints ($a\in[0,50]$,
$b\in[-20,20]$ for Platt; $a,b\in[0,50]$, $c\in[-20,20]$ for beta); isotonic
regression uses the standard pool-adjacent-violators fit with outputs clipped
to $[10^{-6},1-10^{-6}]$; the IW clipping quantile is fixed a priori at
$0.95$. All maps are fit separately for every judge and source domain (the
multi-domain variant pools the fit splits of all domains except the target).
Raw scores are clamped to $[10^{-6},1-10^{-6}]$ identically across domains
before any logit transform. All analysis code, the fixed seed (0), and the
exact split assignments ship with the artifact
(\texttt{code/p34\_full\_analysis.py}).

\section{Reliability and Shift Diagnostics}
\label{app:reliability}
Figure~\ref{fig:reliability} shows pooled per-judge reliability diagrams. Although several judges appear reasonably calibrated when domains are aggregated, this pooling can mask substantial domain-specific calibration drift, motivating the domain-level analysis presented next.

Figure~\ref{fig:gap} plots the accuracy gap against the held-out target $\Delta$Brier for all predeclared transfer units. The weak relationship between these quantities illustrates that the accuracy gap provides only limited predictive value for cross-domain recalibration benefit.

\begin{figure*}[h]
\centering
\figureasset[width=0.92\textwidth]{fig_reliability_perjudge.png}
\caption{Per-judge reliability diagrams (raw scores, five equal-mass bins,
pooled over domains). Pooled curves can hide opposite domain-level calibration
errors; the domain-stratified drift underlying them is shown in
Figure~\ref{fig:drift}.}
\label{fig:reliability}
\end{figure*}

\begin{figure*}[h]
\centering
\figureasset[width=0.62\textwidth]{fig_gap_scatter.png}
\caption{Accuracy gap versus held-out target $\Delta$Brier for Platt scaling,
one point per predeclared transfer unit (448 primary transfers), with binned
means. Bootstrap replicates are never plotted as independent configurations.}
\label{fig:gap}
\end{figure*}

\section{Proxy-Label Audit}
\label{app:proxyaudit}

No human annotation is used for SAMSum or TruthfulQA, so we characterize the
automatic label mechanism itself. For SAMSum, 65\% of items fall in
the verifier-fallback overlap zone $[0.15,0.30)$, so most committed SAMSum
labels were decided by the small verifier rather than by ROUGE-L; replacing
them with the hard ROUGE-L rule agrees with the committed labels on only
32.5\% of items and is therefore a severe stress relabeling, not a small
perturbation. For TruthfulQA, 25\% of items fall in the fallback zone and the
hard token-F1 rule agrees with committed labels on 80\% of items. Both facts
support treating these domains as secondary. Under the three threshold
variants $\{0.75,1.0,1.25\}\times$ (no verifier fallback), the all-domain
probe-gap association stays in the $r=0.30$--$0.38$ band with mean Platt
$\Delta$Brier between $-0.095$ and $-0.098$, versus $r=0.29$ on the
deterministic-only subset: the association is stable under label-mechanism
perturbation.

\section{Complete Second-Generator Results}
\label{app:gen2full}

Table~\ref{tab:gen2full} reports the full method suite on the Llama-3.2-1B
replication, parallel to Table~\ref{tab:mainresults}, so that the selection
of replication results in the main text can be audited.

\begin{table*}[h]
\centering
\small
\setlength{\tabcolsep}{3.4pt}
\begin{tabular}{lrrrrr}
\toprule
Method & $\Delta$Brier $\downarrow$ [95\% CI] & $\Delta$log loss $\downarrow$ [95\% CI] &
ECE $\downarrow$ & harm rate $\downarrow$ [95\% CI] & AUROC \\
\midrule
Raw & --- & --- & 0.369 & --- & 0.610 \\
Temperature & $-0.143$ [$-0.214,-0.087$] & $-3.76$ [$-5.05,-2.61$] & 0.203 & 0.08 [0.01, 0.16] & 0.610 \\
Platt, $a\geq0$ & $-0.058$ [$-0.120,-0.015$] & $-1.94$ [$-2.89,-1.12$] & 0.305 & 0.24 [0.10, 0.40] & 0.590 \\
Beta calibration & $-0.051$ [$-0.112,-0.007$] & $-1.65$ [$-2.63,-0.80$] & 0.320 & 0.25 [0.10, 0.41] & 0.605 \\
Isotonic & $-0.113$ [$-0.203,-0.046$] & $-3.25$ [$-4.67,-1.95$] & 0.240 & 0.22 [0.08, 0.36] & 0.592 \\
Multi-domain temp. & $-0.152$ [$-0.222,-0.092$] & $-3.78$ [$-5.07,-2.64$] & 0.192 & 0.09 [0.00, 0.22] & 0.610 \\
IW-Platt & $-0.067$ [$-0.128,-0.022$] & $-2.14$ [$-2.98,-1.38$] & 0.294 & 0.23 [0.09, 0.39] & 0.591 \\
Label-shift EM & $+0.043$ [$-0.026,+0.098$] & $-0.06$ [$-1.20,+0.99$] & 0.424 & 0.59 [0.39, 0.78] & 0.590 \\
Target temp., equal labels & $-0.155$ [$-0.226,-0.095$] & $-3.78$ [$-5.08,-2.63$] & 0.183 & 0.05 [0.00, 0.15] & 0.610 \\
Target Platt, equal labels & $-0.102$ [$-0.170,-0.051$] & $-2.47$ [$-3.46,-1.56$] & 0.245 & 0.11 [0.00, 0.25] & 0.608 \\
\bottomrule
\end{tabular}
\caption{Second-generator (Llama-3.2-1B) calibration transfer on untouched
target-test records: 448 transfer units, all eight primary judges, paired
means with two-way judge/domain cluster-bootstrap 95\% intervals
($B{=}2000$), same protocol and columns as Table~\ref{tab:mainresults}.
Source-fitted transfer is milder here than on the primary generator (lower
harm rates), the EM label-shift correction is again the only harmful method,
and equal-budget target temperature remains the best row.}
\label{tab:gen2full}
\end{table*}

\section{Parse-Failure Sensitivity}
\label{app:parsefail}

Verbalized scoring occasionally fails to parse and the committed protocol
imputes the predeclared $0.5$ without logging which items failed. We
re-scored both generators with a per-item failure flag under the verbatim
protocol; the rescores are bit-identical to the committed scores
(correlation $1.0$, identical fraction $1.0$ on both generators), which
simultaneously establishes reproducibility and supplies the flags. The
previously unlogged primary-run failure rates are $20.6\%$ (Gemma-2),
$14.1\%$ (Phi-3), $0.3\%$ (TinyLlama), and $0$ for the other five judges;
the second-generator rates are $14.4\%$ (Gemma-2), $13.8\%$ (Phi-3), and
$\leq1.3\%$ elsewhere.

\begin{table}[t]
\centering
\small
\setlength{\tabcolsep}{3.6pt}
\resizebox{\columnwidth}{!}{%
\begin{tabular}{llrrrr}
\toprule
Gen. & Treatment & $r_{\mathrm{gap}}$ & $\Delta$B temp & $\Delta$B tgt-t & OLS raw \\
\midrule
Pri. & Impute $0.5$ & 0.253 & $-0.147$ & $-0.171$ & $+0.270$ \\
Pri. & Exclude flagged & 0.271 & $-0.152$ & $-0.177$ & $+0.283$ \\
G2 & Impute $0.5$ & 0.085 & $-0.143$ & $-0.155$ & $+0.242$ \\
G2 & Exclude flagged & 0.093 & $-0.149$ & $-0.163$ & $+0.238$ \\
G2 & Constrained dec. & 0.067 & $-0.136$ & $-0.141$ & $+0.094$ \\
\bottomrule
\end{tabular}}
\caption{Parse-failure sensitivity: headline quantities under the committed
$0.5$ imputation, exclusion of flagged item--judge pairs, and (second
generator) digit-constrained decoding. Point estimates; the seed-0
cluster intervals in the main text apply to the committed rows.}
\label{tab:parsefail}
\end{table}

Table~\ref{tab:parsefail} shows every conclusion is insensitive to the
treatment: exclusion moves means by at most $0.008$ Brier, the target-
temperature harm rate is unchanged ($0.094$ primary, $0.047$ second
generator, under every treatment), and adding the per-cell failure fraction
as a regression covariate leaves the gap and raw-error coefficients
essentially untouched (failure-fraction coefficients $+0.025$ and $-0.003$).
Constrained numeric decoding is not a free repair: it eliminates Phi-3 and
Gemma-2 failures ($44\to1$, $46\to0$ of 320) but \emph{creates} failures for
TinyLlama ($0\to92$) and Zephyr ($0\to112$), whose natural numeric emissions
do not tokenize inside the constrained set, and it perturbs the score
distribution enough to move the raw-error coefficient. There is no uniformly
better elicitation; the committed imputation is not driving any claim.

\section{Prompt and Order Robustness}
\label{app:promptrobust}

Three meaning-preserving confidence prompts (A: committed; B, C:
paraphrases) and a presentation-order control (Aswap: candidate answer
before reference) were scored for all eight primary judges on the primary
generator under the verbatim pipeline, with prompt-paired per-transfer
inference.

\begin{table}[t]
\centering
\small
\setlength{\tabcolsep}{3.6pt}
\resizebox{\columnwidth}{!}{%
\begin{tabular}{lrrrrr}
\toprule
Variant & $r_{\mathrm{gap}}$ & $\Delta$B temp & $\Delta$B tgt-t & harm tgt-t & fail \\
\midrule
A (committed) & 0.253 & $-0.147$ & $-0.171$ & 0.09 & 0.044 \\
B (paraphrase) & 0.317 & $-0.128$ & $-0.141$ & 0.11 & 0.107 \\
C (paraphrase) & 0.287 & $-0.134$ & $-0.152$ & 0.13 & 0.018 \\
Aswap (order) & 0.246 & $-0.165$ & $-0.191$ & 0.09 & 0.041 \\
\bottomrule
\end{tabular}}
\caption{Prompt/order robustness on the primary generator (448 transfers
per variant; point estimates). ``fail'' is the overall parse-failure rate,
which is itself prompt-sensitive (prompt B drives Gemma-2 to $59\%$).}
\label{tab:promptrobust}
\end{table}

The method conclusions are prompt-stable: temperature and target
temperature improve under every variant, target temperature is the best or
tied-best row throughout, and prompt-paired per-transfer deltas against
variant A are small (mean $|\Delta|\leq0.03$ Brier for temperature, Platt,
and target temperature; sign agreement $0.67$--$0.94$). The gap association
stays in the same weak band ($r=0.25$--$0.32$), so prompt choice neither
rescues nor destroys it. The clearest prompt effect is on elicitation
itself: failure rates range from $1.8\%$ to $10.7\%$ across
meaning-preserving paraphrases, reinforcing that verbalized confidence is a
fragile measurement channel. Table~\ref{tab:promptrobust} summarizes prompt and ordering robustness
results across the primary generator.

\section{Split-Seed Sensitivity}
\label{app:seeds}

With 40 items per domain, individual split assignments matter, so we reran
the complete transfer analysis under ten different question-level split
seeds (point estimates per seed; the committed seed-0 run provides the
intervals). The conclusions the paper relies on are seed-stable: on the
primary generator, mean $\Delta$Brier ranges are $[-0.060,-0.019]$ (Platt),
$[-0.168,-0.125]$ (temperature), $[-0.184,-0.141]$ (multi-domain), and
$[-0.186,-0.152]$ (target temperature, always the best or tied-best row,
harm rate $\leq 0.16$ at every seed), and the probe raw-error coefficient
stays positive on both generators ($[0.16,0.27]$ primary, $[0.12,0.27]$
second generator). The gap association is exactly as fragile as the main
text states: primary $r\in[0.21,0.45]$ across seeds (always positive,
never stable), while the second-generator $r\in[-0.08,0.20]$ and its
adjusted slope $\in[-0.25,0.12]$ span zero. Seed~0 is not a favorable draw
for the gap claim; if anything it sits near the middle of these ranges.

\section{Analyses Run and Not Run}
\label{app:checks}

Completed prespecified analyses: (1) deterministic-label-only subset of the
primary transfers (Section~\ref{sec:proxysens}); (2) leave-one-judge-out and
leave-one-target-domain-out prediction of recalibration harm
(Table~\ref{tab:regression} discussion); (3) proxy-label threshold
sensitivity at $\{0.75,1.0,1.25\}\times$ the declared thresholds; (4) the
certificate at target-probe budgets $m\in\{4,8,16\}$ on real audits and
$m\in\{32,\dots,1024\}$ on synthetic populations; (5) the modern 7--9B judge
replication of the full transfer analysis; (6) a single-judge frontier API
probe (Claude Sonnet 4.5) under the verbatim protocol, reported
descriptively; (7) a full second-generator replication (Llama-3.2-1B
responses, byte-identical labeling, all eight primary judges rescored with
recorded revisions); (8) parse-failure sensitivity with per-item flags,
exclusion, indicator, and constrained-decoding arms
(Appendix~\ref{app:parsefail}); (9) prompt-paraphrase and
presentation-order robustness with prompt-paired inference
(Appendix~\ref{app:promptrobust}); (10) the ten-seed split sensitivity
(Appendix~\ref{app:seeds}).

Not run, and therefore claimed nowhere: stochastic-decoding repeats,
alternative gap definitions beyond the absolute gap, debiased/smoothed ECE
variants, unconstrained Platt slope, weight-clipping sweeps, and Holm-adjusted
contrasts across the full method grid. Each omission is either inherited from
the frozen data-collection design or a disclosed scope limit of this study.

\section{Supplementary Figures}
\label{app:sup_fig}
The adjusted regression analysis in Table~\ref{tab:regression} is visualized
as a coefficient forest plot in Figure~\ref{fig:regression_forest}. The
bootstrap confidence intervals show that probe-split raw Brier is the only
predictor with a confidence interval excluding zero, while the remaining
transfer-gap and distribution-shift predictors exhibit substantial
uncertainty.

Figure~\ref{fig:recalibration_performance}(a,b) highlights the improvement--risk trade-off, showing that target-aware temperature methods achieve the strongest gains with the lowest harm rates.
Figure~\ref{fig:certificate_behavior} shows that larger audits improve detection only modestly, with high false acceptance persisting even at $m=1024$.

\begin{figure*}[t]
\centering
\figureasset[width=\textwidth]{fig_regression_forest_color.png}
\caption{
Adjusted predictors of calibrated target-test Brier. Coefficient estimates
and 95\% cluster-bootstrap confidence intervals from the ANCOVA reported in
Table~\ref{tab:regression}. The vertical dashed line denotes zero. The
probe-split raw Brier predictor shows a statistically reliable association,
whereas other predictors have confidence intervals overlapping the null.
}
\label{fig:regression_forest}
\end{figure*}

\begin{figure*}[t]
\centering
\figureasset[width=0.88\textwidth]{fig_recalibration_performance.png}
\caption{
Recalibration performance across methods.
(a) Mean change in target Brier score and
(b) harm rate across recalibration methods reported in
Table~\ref{tab:mainresults}.
Values are computed directly from the 448 primary-transfer results.
Negative $\Delta$Brier indicates improvement over raw confidence scores,
whereas harm rate measures the fraction of transfers where recalibration
degrades target-test Brier.
}
\label{fig:recalibration_performance}
\end{figure*}

\begin{figure*}[t]
\centering
\figureasset[width=\textwidth]{fig_certificate_behavior.png}
\caption{
Certificate behavior as a function of audit size.
Harm-detection power and descriptive false-acceptance rate for the
certificate at the audit sizes reported in Table~\ref{tab:certificate}.
Values are taken directly from the synthetic-population experiments.
Increasing audit size improves detection power only modestly, while false
acceptance remains high even at $m=1024$.
}
\label{fig:certificate_behavior}
\end{figure*}

\end{document}